\newcommand{\commentAGI}[1]{#1}  %
\newcommand{\commentARXIV}[1]{}  

\commentAGI{
\documentclass{article}
}
\commentARXIV{
\documentclass{llncs}  
}

\newcommand{\comment}[1]{} 
\newcommand{\commentvisible}[1]{} 

\commentAGI{
\textwidth = 6.5 in 
\textheight = 9 in 
\oddsidemargin = 0.0 in
\evensidemargin = 0.0 in
\topmargin = 0.0 in
\headheight = 0.0 in
\headsep = 0.0 in
}

\input{glyphtounicode}  
\usepackage[colorlinks,\commentAGI{pagebackref,}linktocpage]{hyperref}

\usepackage{amssymb}  
\usepackage{amsmath}   

\commentAGI{
\usepackage{amsthm} 
\newtheorem{theorem}{Theorem}
\newtheorem{proposition}[theorem]{Proposition}

\newtheorem{definition}{Definition}
}

\usepackage{amsopn}

\DeclareMathOperator*{\argmin}{arg\,min}

\usepackage{color}
\definecolor{orange}{RGB}{255,127,0}
\definecolor{brown}{RGB}{150,70,0}
\definecolor{green}{RGB}{127,255,127}
\definecolor{darkgreen}{RGB}{0,127,0}
\definecolor{blue}{RGB}{127,127,255}
\definecolor{lightblue}{RGB}{150,150,255}
\definecolor{darkblue}{RGB}{0,0,127}
\definecolor{red}{RGB}{255,90,90}
\definecolor{grey}{RGB}{127,127,127}
\definecolor{pink}{RGB}{255,180,180}

\usepackage{xspace}   
\newcommand{\xaxis}{$x$-axis\xspace}
\newcommand{\yaxis}{$y$-axis\xspace}

\newcommand{\Steps}{\ensuremath{S}} 
\newcommand{\Length}{\ensuremath{L}} 
\newcommand{\LS}{\ensuremath{LS}} 
\newcommand{\Response}{\ensuremath{R}} 

\newcommand{\Exp}[1]{\ensuremath{\mathbb{E}[#1]}\xspace}
\newcommand{\Acc}{\ensuremath{\mathbb{A}}\xspace}

\newcommand{\Unif}[1]{\ensuremath{{{\cal{U}}{#1}}}\xspace}
\newcommand{\Psy}{\Psi}
\newcommand{\Psydiff}[1]{\ensuremath{\Psy_{#1}\xspace}}

\newcommand{\One}{\ensuremath{\mbox{\bf\sffamily 1}\xspace}}

\newcommand{\TAU}{}  

\usepackage{graphicx}

\commentAGI{
\title{A note about the generalisation of the C-tests} 
}
\commentARXIV{
\title{C-tests revisited: back and forth with complexity} 
}

\author
{
	Jos\'{e} Hern\'{a}ndez-Orallo\commentAGI{\\
	{\normalsize\em DSIC, Universitat Polit\`ecnica de Val\`encia, Spain}\\
	{\normalsize \tt jorallo@dsic.upv.es}
	}
}
\commentARXIV{
\institute{DSIC, Universitat Polit\`ecnica de Val\`encia, Spain\\  
\email{jorallo@dsic.upv.es}}
}

\commentAGI{
\date{December 31, 2014} 
}

\begin{document}

\maketitle

{
\commentAGI{\abstract In this exploratory note we ask the question of what a measure of performance for all tasks is like if we use a weighting of tasks based on a difficulty function. This difficulty function depends on the complexity of the (acceptable) policy for the task (instead of a universal distribution over tasks or an adaptive test). The resulting aggregations and decompositions are (now retrospectively) seen as the natural (and trivial) interactive generalisation of the $C$-tests.}
\commentARXIV{\abstract We explore the aggregation of tasks by weighting them using a difficulty function that depends on the complexity of the (acceptable) policy for the task (instead of a universal distribution over tasks or an adaptive test). The resulting aggregations and decompositions are (now retrospectively) seen as the natural (and trivial) interactive generalisation of the $C$-tests.}
}

\section{Introduction}\label{sec:intro}

\commentAGI{Since the inception of algorithmic information theory (AIT) in the 1960s, its use to construct intelligence tests was hinted by some and explicitly suggested by \cite{Chaitin1982}. The first actual implementation of a}
\commentARXIV{A first} test using AIT was the $C$-test \cite{HernandezOrallo-MinayaCollado1998,HernandezOrallo2000a}, where the goal was to find a continuation of a sequence of letters, as in some IQ tasks \commentAGI{\cite{Sternberg2000,computermodels-iq2014}}, and in the spirit of Solomonoff's inductive inference problems: ``given an initial segment of a sequence, predict its continuation'' (as quoted in \cite[p.332]{Li-Vitanyi08}). Levin's $Kt$ complexity was used to calculate the difficulty of a sequence of letters. The performance was measured as an aggregated value over a range of difficulties: 
\begin{equation}
I(\pi) \triangleq \sum_{h=1}^H h^e \sum_{i=1}^N \frac{1}{N} \; \mbox{Hit}(\pi,x_{i,h}) \label{eq:ctest}
\end{equation}
%
%
%
\noindent where $\pi$ is the subject, the difficulties range from $h=1$ to $H$ and there are $N$ sequences $x_{i,k}$ per difficulty $h$. 
The function $\mbox{hit}$ returns 1 if $\pi$ is right with the continuation and 0 otherwise.   
If $e=0$ we have that all difficulties have the same weight. The $N$ sequences per difficulty were chosen (uniformly) randomly.

This contrasts with a more common evaluation in artificial intelligence based on average-case performance according to a probability of problems or tasks:
\begin{equation}
\Psy(\pi) \triangleq \sum_{\mu \in M} p(\mu) \cdot \Exp{R(\pi, \mu)} \label{eq:average} 
\end{equation}
\noindent where $p$ is a probability distribution on the set of tasks $M$, and $R$ is a result function of agent $\pi$ on task $\mu$. 
Actually, eq. \ref{eq:average} can also be combined with AIT, in a different way, by using a universal distribution \cite{solomonoff1964,Li-Vitanyi08}, i.e., $p(\mu) = 2^{-K(\mu)}$, where $K(\mu)$ is the Kolmogorov complexity of $\mu$, as first chosen by \cite{Legg-Hutter2007}.

The work in \cite{Legg-Hutter2007} has been considered a generalisation of  \cite{HernandezOrallo-MinayaCollado1998,HernandezOrallo2000a}, from static sequences (predicting a continuation of a sequence correctly) to dynamic environments (maximising rewards in an MDP). 
In this paper we challenge this interpretation and look for a proper generalisation of \cite{HernandezOrallo-MinayaCollado1998,HernandezOrallo2000a} using the notion of difficulty in the outer sum, as originally conceived and seen in eq.~\ref{eq:ctest}. The key idea is the realisation that for the $C$-test the task and the solution were the same thing. This meant that the difficulty was calculated as the size of the simplest programs that generates the sequence, which is both the task and the solution. Even if the complexity of the task and the solution coincide here, it is {\em the complexity of the solution what determines the difficulty of the problem}. 

However, when we move from sequences to environments or other kind of interactive tasks, the complexity of the policy that solves the task and the complexity of the environment are no longer the same. In fact, this is discussed in \cite{HernandezOrallo-Dowe2010,orallo2014JAAMAS}: the complexity of the environment is roughly an upper bound of the complexity of the acceptable policies (any agent that reach an acceptable performance value), but very complex environments can have very simple acceptable policies. In fact, the choice of $p(\mu) = 2^{-K(\mu)}$ has been criticised for giving too much weight to a few environments. Also, it is important to note that the invariance theorem is more meaningful for Kolmogorov Complexity than for Algorithmic Probability, as for the former it gives some stability for values of $K$ that are not very small, but for a probablity it is precisely the small cases that determine most of the distribution mass. In fact, any computable distribution can be approximated (to whatever required precision) by a universal distribution. This means that the choice of $p(\mu) = 2^{-K(\mu)}$ for Eq. \ref{eq:average} is actually a metadefinition, which leads to virtually any performance measure, depending on the UTM that is chosen as reference.



By decoupling the complexity of task and policy \commentAGI{we see that} we can go back to eq. \ref{eq:ctest} and work out a notion of difficulty of environments that depends on the complexity of the policy. While this may look retrospectively trivial, and the natural extension  in hindsight, we need to solve and clarify some issues, and properly analyse the relation of the two different philosophies given by eq.~\ref{eq:average} and eq.~\ref{eq:ctest}.

\commentAGI{The rest of the paper is organised as follows.} 
Section \ref{sec:background} discusses some previous work, introduces some notation and recovers the difficulty-based decomposition of aggregated performance. 
\commentAGI{Section \ref{sec:task} shows what happens if difficulty functions are based on task complexity.} 
Section \ref{sec:difficulty} introduces several properties about difficulty functions and the view of difficulty as policy complexity.
Section \ref{sec:probs} discusses the choices for the difficulty-dependent probability. 
Section \ref{sec:steps} briefly deals with the role of computational steps for difficulty. 
Section \ref{sec:discussion} closes the paper with a discussion.

\section{Background}\label{sec:background}

AI evaluation has been performed in many different ways
\commentAGI{\cite{newell1976computer,
gaschnig1983evaluation,
rothenberg1987evaluating,
geissman1988verification,
cohen1988evaluation,
decker1989evaluating,
langley1987research,
langley2011changing,
buchanan1988artificial,
simon1995artificial,
baldwin1995process,
falkenauer1998method,
langford2005clever,
LeggHutterTests2007,
whiteson2011protecting,
drummond2010warning,
anderson2011robotics,
madhavan2009performance,
schlenoff2011performance
}} (for a recent account of AI evaluation, see \cite{hernandez2014ai}), but a common approach is based on averaging performance on a range of tasks, as in eq. \ref{eq:average}.

\begin{figure}
\centering
{\sffamily\small

$$
\begin{array}{llll}
h=9  & : & $a, d, g, j, ...$ & $Answer: m$ \\
h=12 & : &  $a, a, z, c, y, e, x, ...$ & $Answer: g$ \\
h=14 & : & $c, a, b, d, b, c, c, e, c, d, ...$ & $Answer: d$  \\
\end{array}
$$
}
\vspace{-0.5cm}
\caption{Several series of different difficulties 9, 12, and 14 used in the $C$-test \cite{HernandezOrallo2000a}.}
\label{fig:Ctest}
\vspace{-0.3cm}
\end{figure}

In what follows, we will focus on the approaches that are based on AIT. As mentioned above, 
the first intelligence test using AIT was the so-called $C$-test \cite{HernandezOrallo-MinayaCollado1998,HernandezOrallo2000a}. 
Figure \ref{fig:Ctest} shows examples of sequences that appear in this test. 
The difficulty of each sequence was calculated as Levin's $Kt$, a time-weighted version of Kolmogorov complexity $K$.  
%
Some preliminary experimental results showed that human performance correlated with the absolute difficulty ($h$) of each exercise and also with IQ test results for the same subjects. Figure \ref{fig:Ctest2} shows the results (taken from \cite{HernandezOrallo-MinayaCollado1998,HernandezOrallo2000a}). HitRatio is defined as the inner sum of eq.~\ref{eq:ctest}:
\begin{equation}\label{eq:hitrate}
\mbox{HitRatio}(\pi,h) \triangleq \sum_{i=1}^N \frac{1}{N} \; \mbox{Hit}(\pi,x_{i,h})
\end{equation}
\noindent An interesting observation is that by arranging problems by difficulty we see that HitRatio seems to be very small from difficulty 8 on. This makes the estimation of the measure much easier, as we only need to focus on (the area of) a small interval of difficulties. In fact, this use of difficulty is common in psychometrics.

\begin{figure}
\centering
\vspace{-1.3cm}
\includegraphics[width=0.48\textwidth]{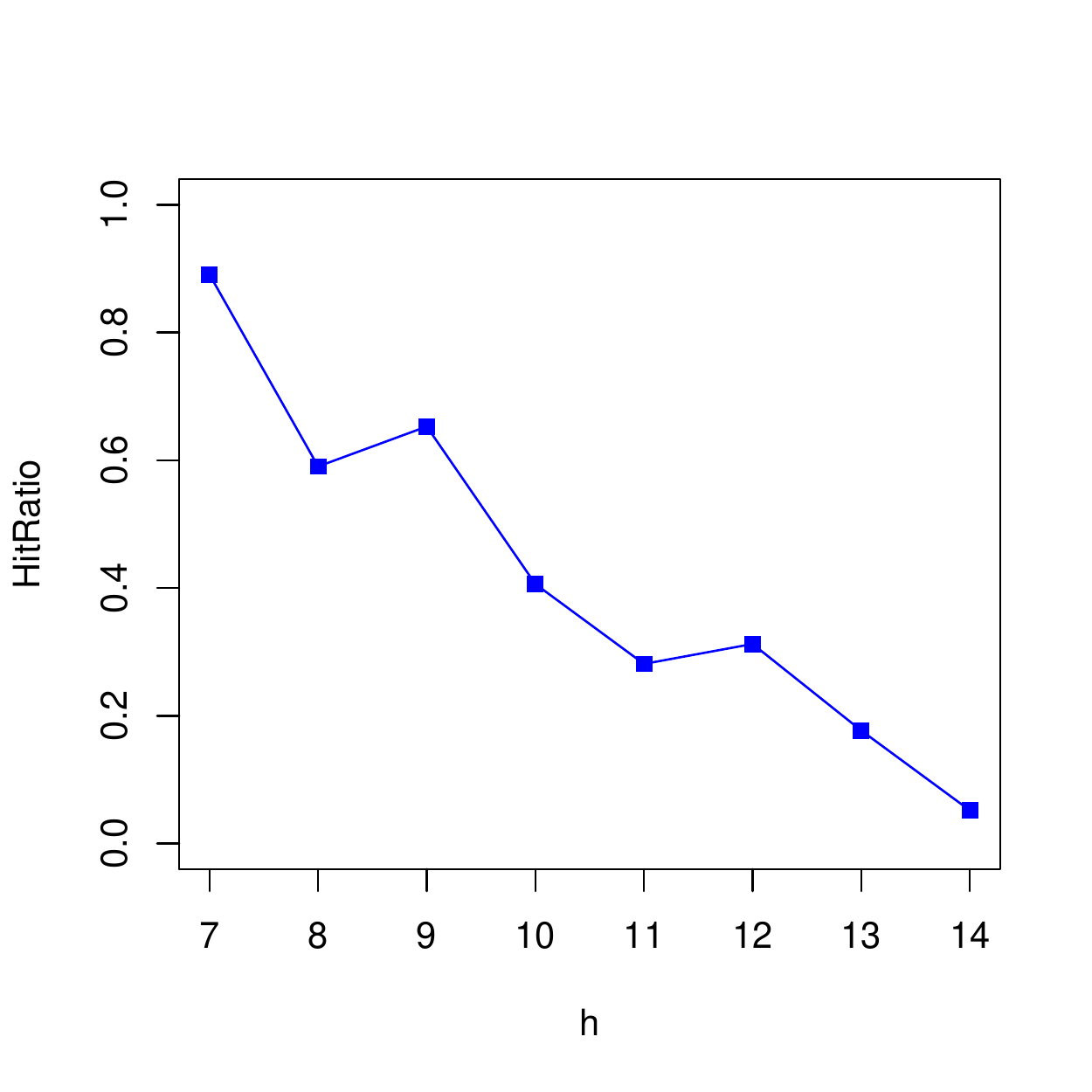} 
\vspace{-0.4cm}
\caption{Results obtained by humans on task of different difficulty in the $C$-test \cite{HernandezOrallo2000a}.}
\label{fig:Ctest2}
\vspace{-0.4cm}
\end{figure}

Several generalisations of the $C$-test were suggested (for ``cognitive agents [...] with input/output devices for a complex environment'' \cite{HernandezOrallo-MinayaCollado1998} where ``rewards and penalties could be used instead'' \cite{HernandezOrallo00b}) or extending them for other cognitive abilities \cite{HernandezOrallo00d}, but not fully developed. 
%
%
%

AIT and reinforcement learning were finally combined in \cite{Legg-Hutter2007}, where all possible environments were considered in 
eq.~\ref{eq:average}, instantiated with a universal distribution for $p$, i.e., $p(\mu) = 2^{-K(\mu)}$, with $K(\mu)$ being the Kolmogorov complexity of each environment $\mu$. 
Some problems (computability, discriminating power, overweight for small environments, time, ...) were discussed with the aim of making a more applicable version of this appraoch by  \cite{Hibbard2009} and \cite[secs. 3.3 and 4]{HernandezOrallo-Dowe2010}. \commentAGI{Some tests were developed \cite{CAEPIA2011,AGI2011Evaluating,AISB-AICAP2012a,leggveness2013approximation}, using the environment class defined in \cite{HernandezOrallo10b} and some ways of aggregated rewards \cite{HernandezOrallo10c}. Despite the limited results, the experiment had quite a repercussion \cite{TheEconomist2011,NewScientist2011,TheConversation2011,yonck2012toward}.}

While the aim of all these proposals was to measure intelligence, many interesting things can happen if AIT is applied to cognitive abilities other than intelligence, as suggested in \cite{HernandezOrallo00d} for the passive case and hinted in \cite[secs. 6.5 and 7.2]{HernandezOrallo-Dowe2010} for the dynamic cases, which proposes the use of different kinds of videogames as environments (two of the most recently introduced benchmarks and competitions in AI are in this direction \cite{bellemare13arcade,Schaul2014}). 
\commentAGI{Also, some hybridisations were also proposed \cite{AGI2011DarwinWallace,AGI2011Compression,AISB-AICAP2012b,manchester2012,AGI2012social,insa2014}, the notion of potential intelligence 
\cite{hernandez2013potential}, 
as well as an integration with the measurement of natural systems under the umbrella of `universal psychometrics' \cite{upsychometrics2} and the notion of universal test \cite{DoweHernandez-Orallo2013a_universal}.}

\commentAGI{
\subsection{Notation}\label{sec:notation}
}

We consider tests that are composed of tasks (also called environments or items) and are performed by agents (also called policies or subjects). 
The set of tasks is denoted by $M$. Its elements are usually denoted by $\mu$.
The set of agents is denoted by $\Pi$. Its elements are usually denoted by $\pi$.
The length in bits of a string is denoted by $\Length(x)$. We can think of a proper encoding of tasks and agents as strings. Given a UTM $U$ we define Kolmogorov complexity as $K_U(y) = \min_{x \::\:U(x)=y} \Length(x)$. We will usually drop the subindex $U$. 
\commentAGI{
The expected\footnote{This has to be `expected' if we consider stochastic environments or agents.} execution steps of $\pi$ when performing task $\mu$ are denoted by $\Exp{\Steps(\pi,\mu,\tau)}$ for a time limit\footnote{The time limit can be understood as the number of transitions, as in MDPs, or can represent a discrete or continuous time in other kind of synchronous or asynchronous modelling of tasks.} $\tau$. For MDPs one possibility is to consider the (expected value of the) maximum steps taken by $\pi$ for any transition \cite{HernandezOrallo-Dowe2010}. If there is a transition before $\tau$ that does not halt, then $\Exp{\Steps(\pi,\mu,\tau)}$ is infinite. The use of stochastic agents and environments makes things more complicated, as any possible transition with non-zero probability before the time limit $\tau$ that does not halt makes the expectation of $\Steps$ infinite.
}
Finally, we have the expected value of the response, score or result of $\pi$ in $\mu$ for a time limit $\tau$ as $\Exp{\Response(\pi,\mu,\tau)}$. 
The value of $\tau$ will be usually omitted. 
The $\Response$ function always gives values between 0 and 1 and we assume it is always defined (a value of $\Response=0$ is assumed for non-halting situations)\commentAGI{\footnote{Note that we do not go into details about how the $\Response$ and $\Steps$ functions are calculated. One common option in reinforcement learning is $\Exp{\Response(\pi, \mu, \tau)} = \Exp{\sum_{i=1}^{\infty} \gamma^i \rho_i(\pi,\mu)}$, where $\rho_i$ is the result of the reward function for iteration $i$. The term $\gamma^i$ is a discount term with $\gamma$ greater than (but usually close to) 1. The use of a discounting factor $\gamma$ makes the approximation possible even in cases where the rewards are divergent. 
Similarly, we could also use a discount function for $\Steps$ as $\Exp{\Steps(\pi,\mu,\tau)} = \Exp{\sum_{i=1..\infty} \gamma^i \sigma_i(\pi,\mu)}$, where $\sigma_i$ gives the execution steps of transition $i$. 
Even with these choices and a discount factor, an infinite sum would make both functions very difficult to estimate, as we could have a transition with a large value of $i$ giving a high value. For the $\Response$ function this is not so critical, as it is bounded and the contribution decays with $i$, but for $\Steps$ this is problematic.}.}   
We also define $\Exp{\LS(\pi,\mu,\tau)} \triangleq \Length(\pi) + \log \Exp{\Steps(\pi,\mu,\tau)}$. Logarithms are always binary. 
\commentAGI{
Inspired by Levin's $Kt$ (see, e.g., \cite{Levin73} or \cite{Li-Vitanyi08}), we can define $Kt(\mu \TAU) = \min_{\pi} \Exp{\LS(\pi,\mu \TAU)}$, which in this case depends on $\mu$ as well\footnote{
If $\LS$ is considered as a measure of {\em effort}, one can argue that space (i.e., memory) could also be considered, as this has been usually important for the analysis of algorithms and in cognitive science (working memory). However, execution steps are always equal or greater than working memory, so if we consider $\Steps$ there is no need to also include working memory, unless 
we are especially interested in giving relevance (and weight) to working memory.}.}

\commentvisible{
Effort, Cost or LST for an agent $\pi$ on a task $\mu$ up to steps $n$, over descriptional 
machine $U$. We consider running time (steps), working space (bits) and program length (bits).
Program length, working memory size and computation steps. 
For TM, a computation step usually implies moving left or right on a type, so it could be 
measured in bits as well. If there are n states on a TM (e.g., the smallest UTM has 15 states 
and 2 symbols), then each step actually encodes log N bits
\[ LST_U(\pi, \mu, n) = \mathbb{E} [ L(Length(\pi)) + S(\max_{t \leq n}(Space_t(U,\pi,\mu))) + T(\max_{t \leq n}(Time_t(U,\pi,\mu))) ]  \]
where L, S, T are increasing functions. It is better to talk amb Memory instead of Space and 
(Internal) Steps instead of Time, so we would have LMS (or SML).
If $L(x)=x$ (identity), and $S$ and $T$ is LOG
\[ LST_U(\pi, \mu, n) = \mathbb{E} [ Length(\pi) + LOG \max_{t \leq n}(Space_t(U,\pi,\mu)) + LOG \max_{t \leq n}(Time_t(U,\pi,\mu)) ]  \]
If the agent doesn't halt for some input then it has infinite LST. Even if we consider 
stochastic environments and agents, it is sufficient to have a case where it doesn't halt to 
have an infinite value of LST.
(Space could be eliminated, as it is always lower to or equal than time).
So we have more or less $Kt_{max}$ as in the  AIJ paper (definition 9).
\[C(\pi,\mu,n)\]
$C$ or $E(\pi| [\mu])$ is defined as the LST  of describing $\pi$ given $\mu$ (not its code). For instance, "output 12345" is easier to describe if $\mu$ outputs 12345, as "output what $\mu$ outputs" or "follow the indications" or even "try a sequence of numbers until you get good rewards".
Note that $C(\pi_{rand},\mu,n)$ is low. We can get rid of $\mu$ but then we would have that C(123451234512345....) would be the same as C(12745123451234812345...), i.e., the same with noise. The first is just "take the first five observations and repeat them". The second one would have to opt by "repeat 12345".
Also, the use of S in LST may be important, as the second one requires much more space to see the pattern, although if it is "repeat 12345" then it does not count in it. But if the environment also generates the pattern stochastically (i.e., it is not always 12345, with or without noise), then S in LST is important.
}


\commentAGI{
\subsection{Difficulty-based decomposition}
}

It is actually in \cite{upsychometrics2}, where we can find a first connection between the schemas of eq.~\ref{eq:ctest} and  eq.~\ref{eq:average}.
We adapt definition 14 in \cite{upsychometrics2}, which is a generalisation of eq.~\ref{eq:average}, by making the set $M$ and the task probability $p$ explicit as a parameters. 


\begin{definition}\label{def:psi} The expected average result 
for a task class $M$, a distribution $p$  and an agent $\pi$ is:
\begin{eqnarray}\label{eq:psi}
\Psy(\pi,M,p \TAU) \triangleq \sum_{\mu \in M} p(\mu) \cdot \Exp{\Response(\pi, \mu \TAU)}  
\end{eqnarray}
%
\end{definition}

\noindent And now we use proposition 4 in \cite{upsychometrics2} that decomposes it. First, we define partial results for a given difficulty $h$ as follows:
\begin{eqnarray}\label{eq:U}
 \Psydiff{h}(\pi,M,p \TAU) \triangleq  \sum_{\mu \in M, \hbar(\mu)= h} p(\mu|h) \cdot \Exp{\Response(\pi, \mu \TAU)} 
\end{eqnarray}
%
Where $\hbar$ is a difficulty function $\hbar: M \rightarrow \mathbb{R}^+ \cup \: 0$. 
Note that this parametrises the result of eq.~\ref{eq:psi} for different difficulties. For instance, for two agents $\pi_A$ and $\pi_B$ we might have that $\Psydiff{3}(\pi_A) < \Psydiff{3}(\pi_B)$ but $\Psydiff{7}(\pi_A) > \Psydiff{7}(\pi_B)$.
%
%
If we represent $\Psydiff{h}(\pi,M,p \TAU)$ on the \yaxis versus $h$ on the \xaxis we have a so-called agent response curve, much like Fig.~\ref{fig:Ctest2}. \commentAGI{In Fig.~\ref{fig:arc} we show an example of two agent response curves.
\begin{figure}
\centering
\vspace{-0.5cm}
\includegraphics[width=0.45\textwidth]{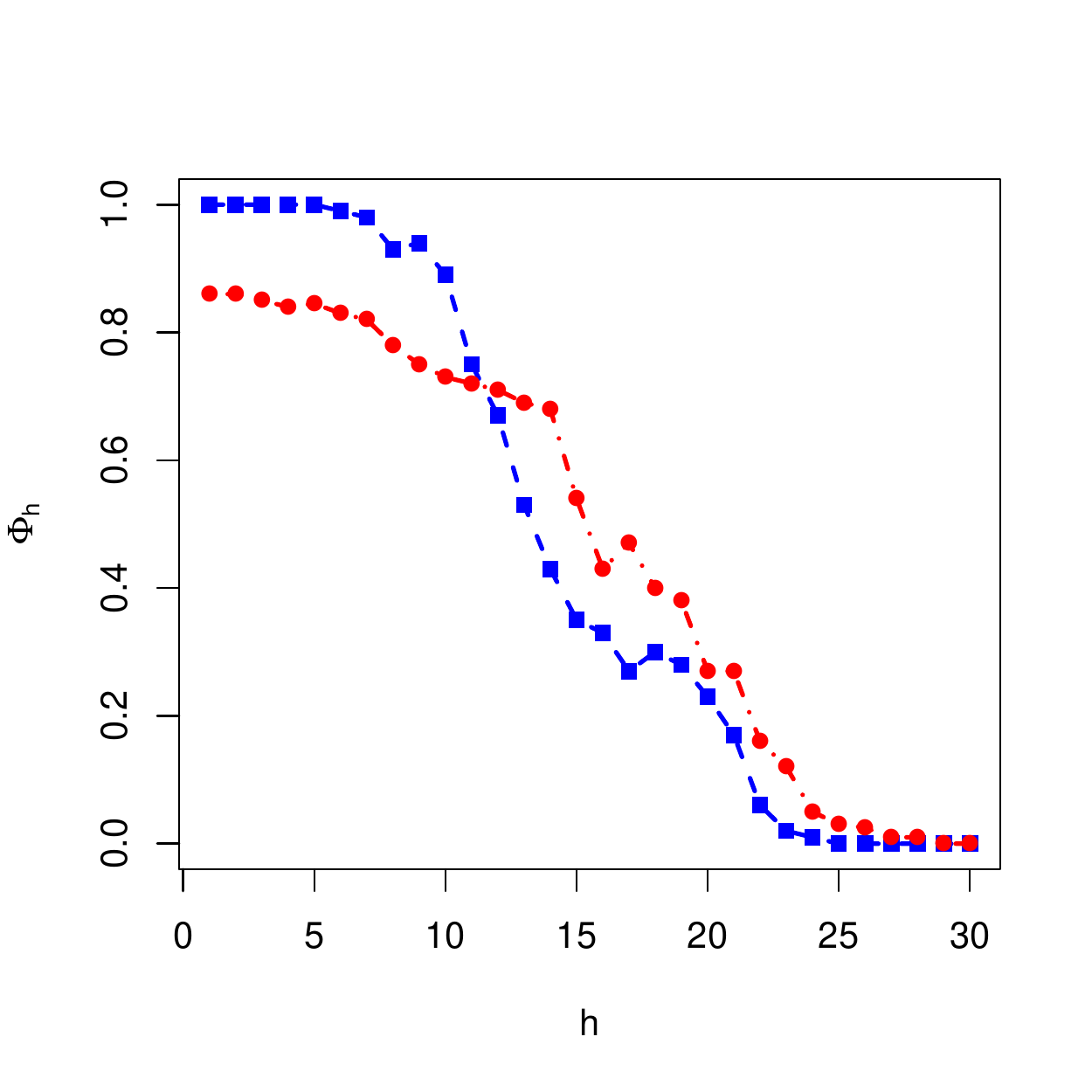}
\vspace{-0.2cm}
\caption{Two figurative agent response curves. Despite having equal areas, one is better for easier tasks while the other is better for more difficult tasks. What is more relevant (to fail in easy tasks or to excel in difficult tasks) is a matter of what we may be interested in when evaluating the systems.}
\label{fig:arc}
\vspace{-0.2cm}
\end{figure}
}

If we want to get a single number from an agent response curve we can aggregate performance for a range of difficulties, e.g., as follows: 


\begin{proposition}\label{prop:psidecomp} (\cite[prop. 4]{upsychometrics2})
The expected average result $\Psy(\pi,M,p \TAU)$ can be rewritten as follows: 
\commentAGI{
\begin{eqnarray}\label{eq:psi2-cont}
\Psy(\pi,M,p \TAU) = \int_{0}^{\infty} p(h) \Psydiff{h}(\pi,M,p \TAU) dh
\end{eqnarray}
or,} in the particular case when $\hbar$ only gives discrete values:
\begin{eqnarray}\label{eq:psi2-disc}
\Psy(\pi,M,p \TAU) = \sum_{h=0}^{\infty} p(h) \Psydiff{h}(\pi,M,p \TAU)
\end{eqnarray}
\end{proposition}
\noindent where $p(h)$ is a \commentAGI{probability density function for eq.~\ref{eq:psi2-cont} and a} discrete probability function for eq.~\ref{eq:psi2-disc}. 
Note that equations \ref{eq:psi}, \ref{eq:U} and \ref{eq:psi2-disc} are generalisations, respectively, of equations \ref{eq:average}, \ref{eq:hitrate} and \ref{eq:ctest}.

\commentAGI{

\section{Difficulty as task complexity}\label{sec:task}

We are now considering the decomposition of definition \ref{def:psi} by using two very particular choices. 
The first choice is that we consider $p(\mu) = 2^{-K(\mu)}$. This is exactly what has been considered in \cite{Legg-Hutter2007,leggveness2013approximation}.
The second choice is that we consider that the difficulty of an environment is $\hbar(\mu)= K(\mu)$. In this very particular case\footnote{In \cite[prop. 3]{upsychometrics2}, it is shown that with these two choices the average difficulty diverges. This does not mean that eq.~\ref{eq:psi2-disc} diverges, which is what we work out here.}, we see that $\hbar$ is discrete ($K$ is defined on the natural numbers). 

With these two choices ($p(\mu)=2^{-K(\mu)}$ and $\hbar=K(\mu)$), the first thing that we realise is that {\em for those $\mu$ such that  $\hbar(\mu)= h$} we have that $p(\mu|h)= \frac{1}{N_M(h)}$, where $N_M(h) \triangleq |\{\hbar(\mu)= h\}|$. In other words, given a difficulty $h$, all tasks $\mu$ of that difficulty have the same $K$ and clearly the same probability, so we have a uniform distribution. The denominator, $N_M(h)$, is (much) lower than $2^h$ as we are using a prefix code and some (many) sequences of length $h$ are not self-delimited programs. 
The value $p(h)$ is then given by $\sum_{\mu \in M, \hbar(\mu)= h} 2^{-K(\mu)} = \sum_{\mu \in M, \hbar(\mu)= h} 2^{-h} = N_M(h)2^{-h}$.
So we have:
\begin{eqnarray*}
\Psy(\pi,M \TAU) & = & \sum_{h=0}^{\infty} N_M(h)2^{-h} \sum_{\mu \in M, \hbar(\mu)= h}  \frac{1}{N_M(h)}  \Exp{\Response(\pi, \mu \TAU)}   
\end{eqnarray*}
%
%
%
\noindent Even if these are very particular (and arguable unreasonable) choices of task probability and task difficulty, this straightforward transformation is helful to show that what we have is that all tasks of the same difficulty have the same weight (the inner sum) but the outer sum can go from a geometric distribution for difficulties (easy tasks are much more likely) if the prefix coding is a unary coding (so $N_M(h)= 1$, and  $p(h) = 2^{-h}$) to much more efficient codings where $N_M(h)$ is close to $2^{h-m}$ where $m$ is a constant that depends on the efficiency of the coding. For instance, for a prefix coding that uses a fixed-length word of size $c$ (character size in bits) and a special character as end delimiter, we would have $N_M(h)=(2^c-1)^{(\frac{h}{c}-1)}$ programs when $h$ is divisible per $c$ and 0 otherwise. In this case $p(h) = N_M(h)2^{-h}$ can be approximated by $\frac{1}{c \cdot \nu}b^{-h}$, where $b = 1+2^{-c}$ and $\nu$ is a constant normalisation term such that $\sum w(h) \leq 1$. This is what we see on Figure \ref{fig:coding} (left) with $c=4$ (so $b=1.0625$). However, we see that $w(h)$ can have a base as close to 1 as we like, so the probability would look uniform for small values of $h$, as Figure \ref{fig:coding} (right). Nonetheless, it still is a geometrical distribution (otherwise the sum would diverge, which is not possible as we assume that $\Response$ is bounded as for equation \ref{eq:psi}). 

\begin{figure}
\centering
\includegraphics[width=0.48\textwidth]{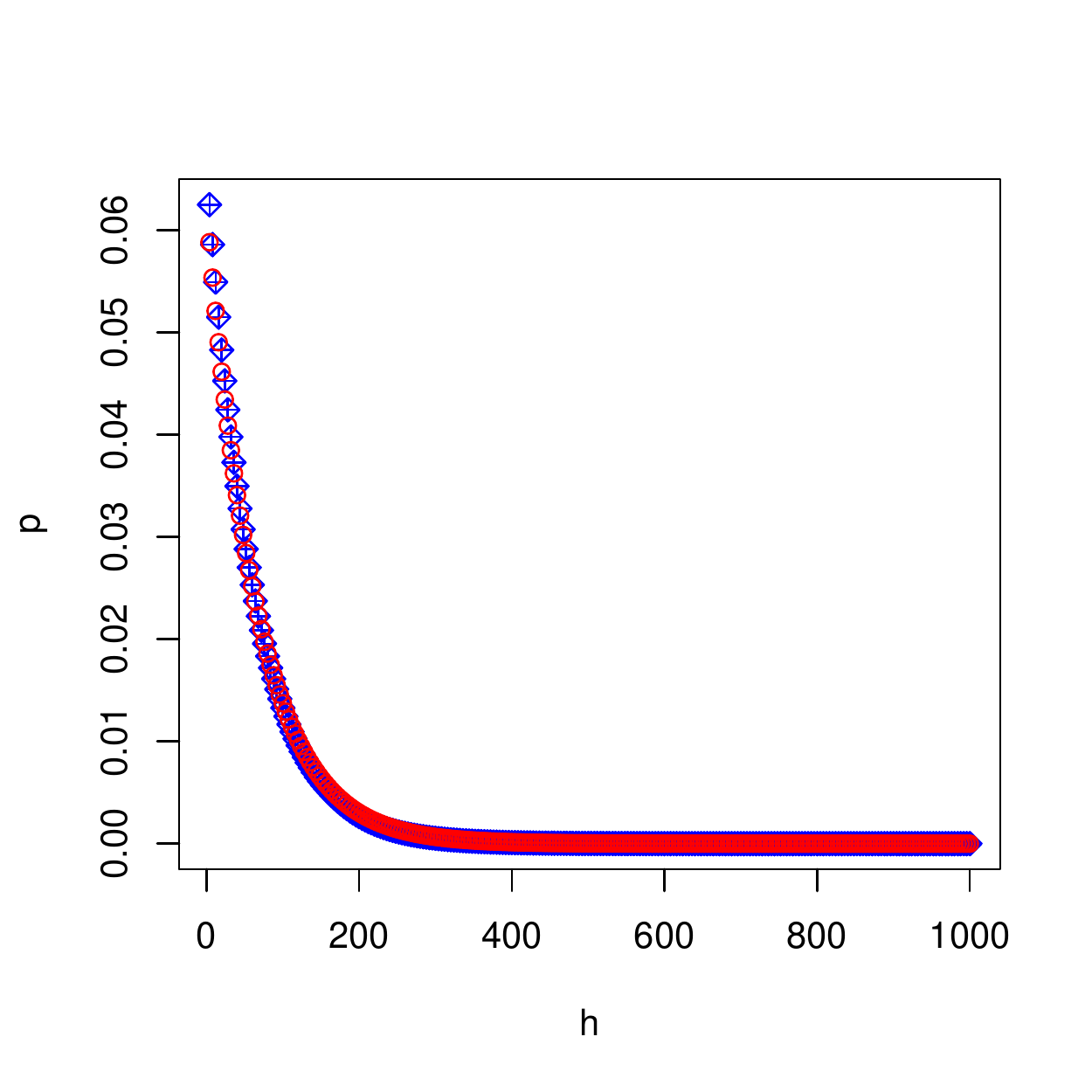} \hfill
\includegraphics[width=0.48\textwidth]{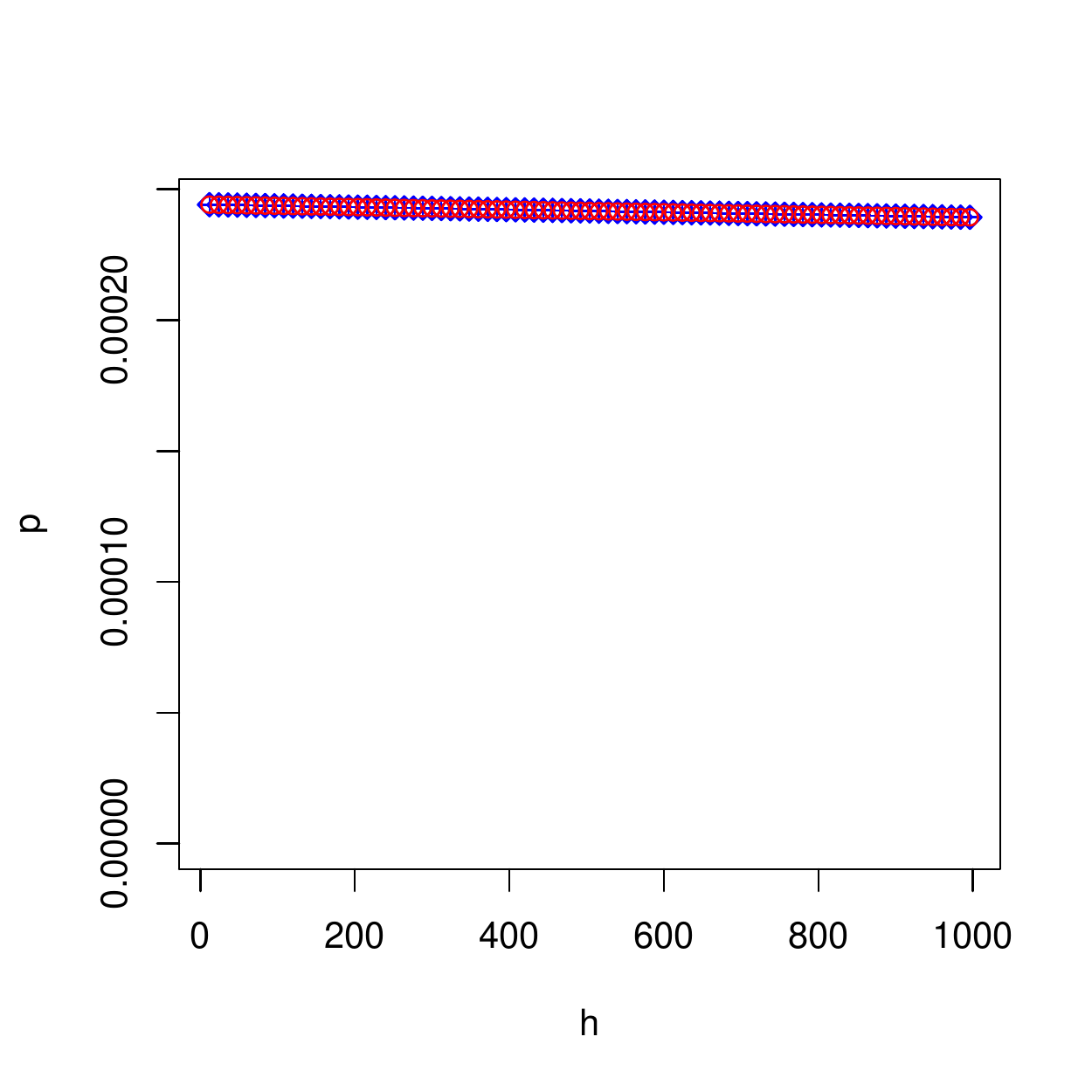} 
\vspace{-0.3cm}
\caption{Left: $p(h) = N_M(h)2^{-h}$ with a character coding with one end delimiter with word length $c=4$. Right: $p(h)$ with word length $c=12$. In blue the true values and in red the approximation $\frac{1}{c \cdot \nu}b^{-h}$, where $b = 1+2^{-c}$ ($b=1.0625$ on the left and $b=1.000244$ on the right). The distribution on the right looks almost uniform but it still geometric.}
\label{fig:coding}
\end{figure}

All this highlights how important the prefix coding is for the understanding of these distributions. In fact, it is not only that the definition depends strongly on the choice of the reference UTM used for $K$, which determines the probability of each task $2^{-K(\mu)}$, as argued elsewhere (\cite{Hibbard2009} and \cite[secs. 3.3 and 4]{HernandezOrallo-Dowe2010}), but also that the prefix coding is highly relevant as well if we use $K(\mu)$ as difficulty.

Apart from the differences that appear because of the coding, there is a major concern about this: it is now crystal clear that easier problems would have more weight than difficult ones. 
However, our intuition when evaluating a subject on a range of difficulties would be to give more relevance to more difficult problems or, at most, to give the same relevance to all difficulties. In the latter case, a subject being able to score well in all tasks of difficulty $2h$ would double the result than another subject being to score well in all tasks of difficulty $h$.

} 

\section{Difficulty functions}\label{sec:difficulty}

Before setting an appropriate measure of difficulty based on the policy, in this section we will analyse which properties a difficulty function may have.

\commentAGI{
\subsection{Detaching $p(h)$ from $p(\mu|h)$}
}

The decomposition in previous section suggests that we could try to fix a proper measure of difficulty first and then think about a meaningful distribution $p(h)$. Once this is settled, we could try to find a distribution for all environments of that difficulty $p(\mu|h)$. In other words, once we determine how relevant a difficulty is we ask which tasks to take for that difficulty. This is the spirit of the $C$-test \cite{HernandezOrallo-MinayaCollado1998,HernandezOrallo2000a} as seen in eq.~\ref{eq:ctest}. In fact, we perhaps we do not need a $p(h)$ that decays dramatically, as it is expectable to see performance to decrease for increasing difficulty, as in Figure~\ref{fig:Ctest2}.

To distinguish $p(h)$ and $p(\mu|h)$ we will denote the former with $w$ and the latter with $p_M$. We will use any distribution or even a measure (not summing up to one, for reasons that we will see later on) as a subscript for $\Psy$. 
For instance, we will use the following notation $\Psy_{\Unif{(h_{min},h_{max})}}(\pi,M,p_M \TAU)$, where $\Unif{(a,b)}$ represents a uniform distribution between $a$ and $b$. For instance, we can have two agents $\pi_A$ and $\pi_B$ such that $\Psy_{\Unif{(1,10)}}(\pi_A) > \Psy_{\Unif{(1,10)}}(\pi_B)$ but $\Psy_{\Unif{(11,20)}}(\pi_A) < \Psy_{\Unif{(11,20)}}(\pi_B)$. 
We will use the notation $\Psy_{\oplus}(\pi,M,p_M \TAU)$ when $w(h)=1$ (note that this is not the uniform distribution for discrete $h$), which means that the partial aggregations for each difficulty are just added. 
In other words, $\Psy_{\oplus}(\pi,M,p_M \TAU) \triangleq \sum_{h=0}^{\infty}  \Psydiff{h}(\pi,M,p_M \TAU)$ for discrete difficulty functions \commentAGI{and   $\Psy_{\oplus}(\pi,M,p_M \TAU) \triangleq \int_{0}^{\infty} \Psydiff{h}(\pi,M,p_M \TAU) dh$ for continuous difficulty functions}. 
We will explore whether this (area under the agent response curve) is bounded.

Figure~\ref{fig:three-approaches} shows approach A, which has already been mentioned, while approaches B and C will be seen in sections \ref{sec:probs1} and \ref{sec:probs2} respectively.

\begin{figure}
\centering
\vspace{-0.5cm}
\includegraphics[width=0.8\textwidth]{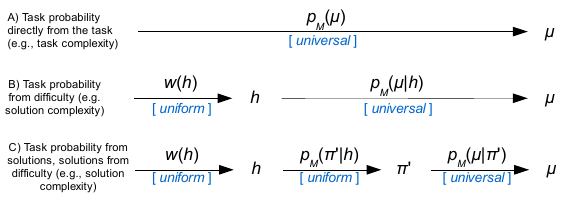}
\vspace{-0.2cm}
\caption{Three approaches to aggregate the results for a set of tasks. Top (A) shows the classical approach of choosing a probability for the task, according to the properties of the task. Middle (B) shows the approach where we arrange tasks by difficulty, and the notion of difficulty is derived from the properties of the policy. Bottom (C) shows a variation of B where we derive acceptable policies for a given difficulty and then generate tasks for each policy. Between square brackets some choices we examine in this paper.}
\label{fig:three-approaches}
\vspace{-0.2cm}
\end{figure}

\commentAGI{
\subsection{Acceptable policies}
}

When we aggregate environments with different scales on $\Response$ and different difficulties, we may have that an agent focusses on a few environments with high difficulty while another focusses on many more environments with small responses. Agent response curves in \cite{upsychometrics2}, which are inspired by item response curves in psychometrics (but inverting the view between agents and items), allow us to see how each agent performs for different degrees of difficulty. Looking at Figure~\ref{fig:Ctest2} and similar agent response curves in psychometrics, we see that the notion of difficulty must be linked to $\Response$, i.e., how well the agents perform, and not about the complexity of the task, as in the previous section. 

\commentAGI{
One first idea of a difficulty function of an environment is the expected response for a random agent $\pi_{rand}$, i.e., $\hbar(\mu) \triangleq \Exp{\Response(\pi_{rand}, \mu \TAU)}$. However, this result is not very meaningful, but just sets some kind of baseline, which will be between $0$ and $1$. This is shown in Figure~\ref{fig:envdiff1} (left) and has the advantage of being applicable to infinite populations of agents $\Pi$ without assuming any distribution of agents. A related idea would be to calculate the expected response for a population of agents, using some distribution, e.g., $\hbar(\mu) \triangleq \sum_{\pi} p_{\Pi}(\pi) \Exp{\Response(\pi, \mu \TAU)}$. As agents are programs, we may have the temptation of using $p_{\Pi}(\pi) = 2^{-K(\pi)}$. If we look at this for an illustrative example in Figure~\ref{fig:envdiff1} (right), we see that this would mean that basically the few results on the lefmost part of the plot will dominate, i.e., the result of the shortest program.

\begin{figure}
\centering
\vspace{-0.2cm}
\includegraphics[width=0.43\textwidth]{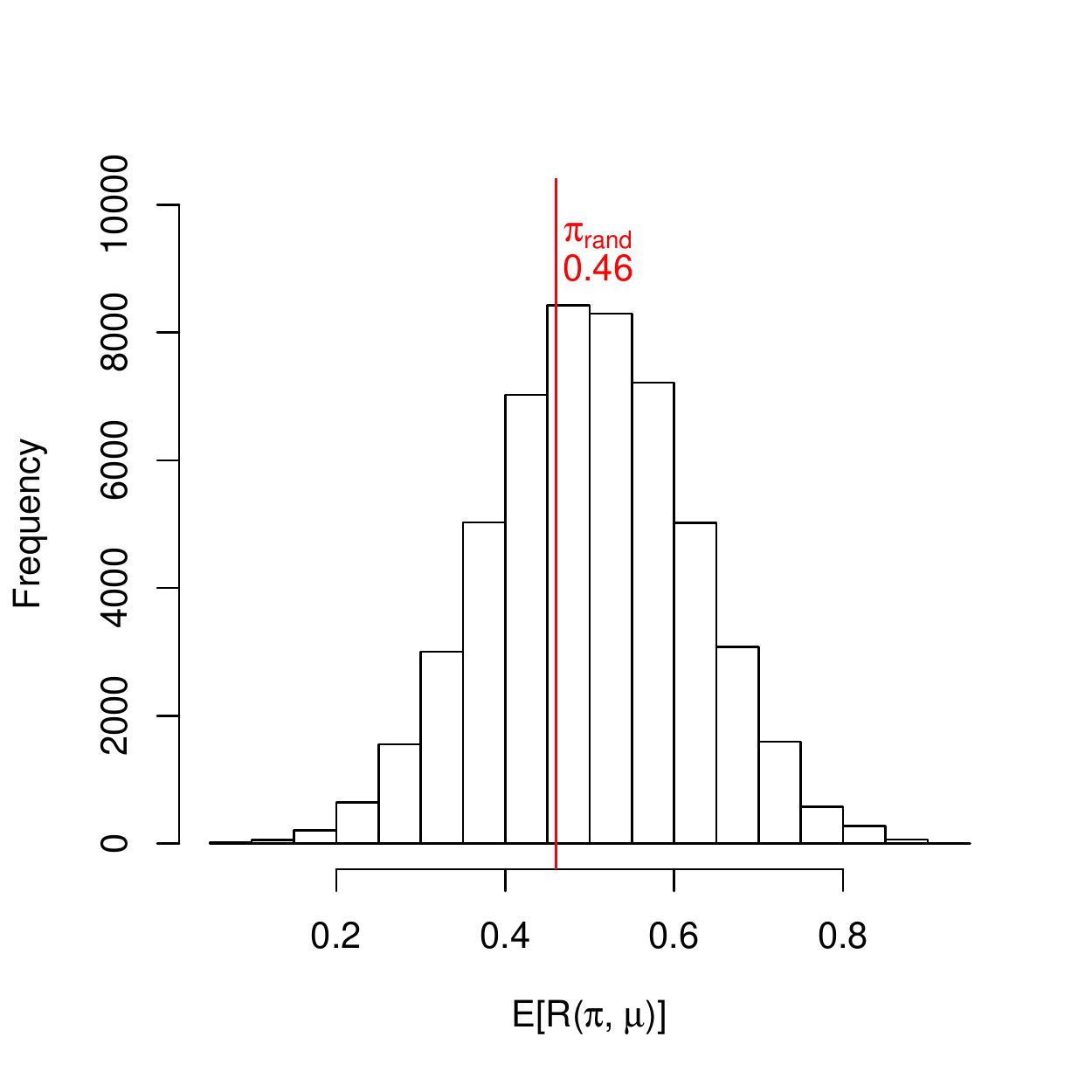}\vspace{0.5cm} \hfill 
\includegraphics[width=0.46\textwidth]{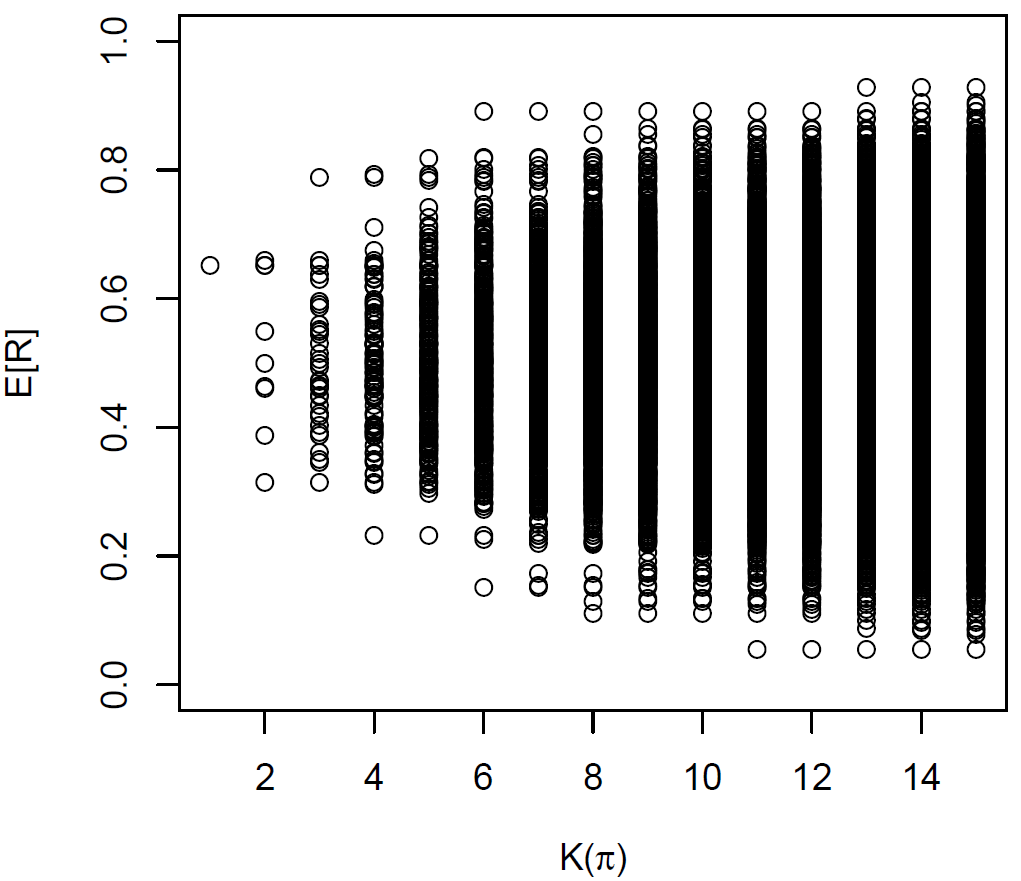} 
\vspace{-0.6cm}
\caption{Left: an illustrative distribution of responses of a population of agents for a single environment, assuming a finite number of agents. A random agent $\pi_{rand}$ has expected response of 0.46. Right: the distribution of responses ordered by $K(\pi)$ on the \xaxis.}
\label{fig:envdiff1}
\vspace{-0.2cm}
\end{figure}

As we want a short policy with a good result, we can think about a ratio between complexity and response of the policy, something like $\hbar(\mu) \triangleq \min_\pi \frac{\Length(\pi)}{\Exp{\Response(\pi,\mu \TAU)} }$. The unit for difficulty would be bits per reward unit. It is not clear why doubling responses should take double bits (perhaps logarithms could be used instead). In the example in Figure~\ref{fig:envdiff1} (right), we see that the leftmost policy ($\Length(\pi)=1$) would have $h= \frac{1}{0.42}=2.38$, and the best with $\Length(\pi)=2$ would have $h= \frac{2}{0.78}=2.56$. In fact, no other policy could get lower than in 2.38. Choosing any other reasonable function instead of a ratio such that we can go further right on the plot would be rather arbitrary. 
} 
Another option is what is done in \cite{orallo2014JAAMAS}, as $\hbar(\mu) \triangleq \min_{\pi : {\Exp{\Response(\pi,\mu \TAU)} = \Response_{max}(\mu \TAU)}} \Length(\pi)$ where $\Response_{max}(\mu \TAU) = \max_\pi {\Exp{\Response(\pi,\mu \TAU)}}$. 
However, 
$\Response_{max}$ may be hard to calculate and even if it can be effectively calculated, any minor mistake or inefficiency in a very good agent will prevent the agent from reaching the optimal result, leading to a notion of difficulty linked to the complexity of the `perfect' policy. In \cite{orallo2014JAAMAS}, a `tolerance value' is considered and, instead of one policy, difficulty is linked to the probability of finding a policy under this tolerance\commentAGI{ by using different search approaches}.

\commentAGI{Here we are going to consider a similar approach.} We are going to consider this tolerance $\epsilon$ of acceptability. 
\begin{equation}\label{eq:accept}
\Acc^{[\epsilon]}(\pi, \mu \TAU) \triangleq \One (\Exp{\Response(\pi, \mu \TAU)} \geq 1- \epsilon )
\end{equation}
This returns 1 if the expected response is above $1-\epsilon$ and 0 otherwise. If $\Acc^{[\epsilon]}(\pi, \mu \TAU) = 1$ we say that $\pi$ is $\epsilon$-acceptable.
With this, we binarise responses. One can argue that we could have just defined a binary $\Response$, but it is important to clarify that it is not the same to have tolerance for each single $\Response$ (or a binarised $\Response$) than to have a tolerance for the expected value $\Exp{\Response}$. The tolerance on the expected value allows the agent to have variability in their results (e.g., stochastic agents) provided the expected value is higher than the tolerance. Finally, even if we will be very explicty about the value of $\epsilon$, and changing it will change the difficulty value of any environment, it is important to say that this value is not so relevant. The reason is that for any environment we can build any other environment where the responses are transformed by any function. In fact, we could actually consider one fixed threshold, such as 0.5, always.

And now we can just define a new version of eq.~\ref{eq:U} using this new function:
\begin{eqnarray}\label{eq:UAcc}
 \Psydiff{h}^{[\epsilon]}(\pi,M,p_M \TAU) \triangleq  \sum_{\mu \in M, \hbar(\mu)= h} p_M(\mu|h) \cdot \Acc^{[\epsilon]}(\pi, \mu \TAU) 
\end{eqnarray}
\noindent We can just rewrite equations \commentAGI{\ref{eq:psi2-cont} and} \ref{eq:psi2-disc} accordingly:
\commentAGI{
\begin{eqnarray}\label{eq:psi2-cont2}
\Psy_w^{[\epsilon]}(\pi,M,p_M \TAU) = \int_{0}^{\infty} w(h) \Psydiff{h}^{[\epsilon]}(\pi,M,p_M \TAU) dh
\end{eqnarray}
}
\begin{eqnarray}\label{eq:psi2-disc2}
\Psy_w^{[\epsilon]}(\pi,M,p_M \TAU) = \sum_{h=0}^{\infty} w(h) \Psydiff{h}^{[\epsilon]}(\pi,M,p_M \TAU)
\end{eqnarray}

\commentAGI{
\subsection{Properties}
}

\commentARXIV{\noindent}Given the above, we are now ready for a few properties about difficulty functions.

\begin{definition}\label{def:boundeddiff} 
A difficulty function $\hbar$ is strongly bounded in $M$ if for every $\pi$ there is a difficulty $h$ such that for every $\mu \in M : \hbar(\mu) \geq h$ we have $\Acc^{[\epsilon]}(\pi, \mu \TAU) = 0$. 
\end{definition}

\commentAGI{
A weaker version of the above property would be as follows:
\begin{definition}\label{def:weaklyboundeddiff} 
A difficulty function $\hbar$ is weakly bounded in $M$ if for every $\pi$ there is a difficulty $h$ such that for every $h' \geq h$ we have $\Psydiff{h'}^{[\epsilon]}(\pi,M,p_M \TAU) = 0$. 
\end{definition}

Clearly, strongly boundedness implies weakly boundeness, but not vice versa. There can be cases that are not strongly bounded but are weakly bounded. For instance, if we just choose $p_M(\mu|h) = 0$ if $h > h'$ for a given $h'$, which is not actually because of a difficulty function, so this is show that weakly boundness is rather a property of the evaluation setting and not only about the dificulty function.




\begin{proposition}\label{prop:boundeddiff}
If a countable difficulty function $\hbar$ is weakly bounded by $\Response$ then for every possible agent $\pi$ we have that:
\begin{eqnarray}\label{eq:psisum}
\Psy_{\oplus}^{[\epsilon]}(\pi,M,p_M \TAU) = \sum_{h=0}^{\infty}  \Psydiff{h}^{[\epsilon]}(\pi,M,p_M \TAU)  
\end{eqnarray}
is finite.
\end{proposition}
\begin{proof}
As it is weakly bounded then there is an $h$ such that for every $h' \geq h$ we have $\Psydiff{h'}^{[\epsilon]}(\pi,M,p_M \TAU) = 0$. Consequently, $\sum_{h=0}^{\infty}  \Psydiff{h}^{[\epsilon]}(\pi,M,p_M \TAU)  = \sum_{h=0}^{h'}  \Psydiff{h}^{[\epsilon]}(\pi,M,p_M \TAU)$.  
As $\hbar$ is countable, and $\Psydiff{h}$ is bounded, then eq.~\ref{eq:psisum} is finite.  
\end{proof}

The above is a modification of eq.~\ref{eq:psi2-disc2} where the probability $p(h)$ has been removed. When the above proposition holds we have that the area under the agent response curve is finite. 
We can obtain a similar result for the continuous case (eq.~\ref{eq:psi2-cont2}), using an integral instead.

}

\commentAGI{
\subsection{Difficulty as policy complexity}
}

Now \commentAGI{we are ready to ask what happens if} we choose the difficulty function in terms of $\epsilon$-acceptability, i.e.:
%
%
%
\begin{equation}\label{eq:hbarK}
\hbar^{[\epsilon]}(\mu) \triangleq \min \{ \Length(\pi) : \Exp{\Response(\pi, \mu \TAU)} \geq 1- \epsilon \} = \min \{ \Length(\pi) :  \Acc^{[\epsilon]}(\pi, \mu \TAU) = 1 \}
\end{equation}
\noindent We can say a few words about the cases where a truly random agent gives an acceptable policy for an environment. If this is the case, we intuitively consider the environment easy. So, in terms, of $\Length$, we consider truly random agents to be simple, which is more reasonable than considering them of infinite difficulty, and goes well with our consideration of stochastic agents and environments. 

Figure \ref{fig:envdiff} (left) shows the distribution of response according to $\Length(\pi)$, 
but setting $\epsilon=0.9$. We see that the simplest $\epsilon$-acceptable policy has $\Length=12$. 

\begin{figure}
\centering
\vspace{-0.3cm}
\includegraphics[width=0.45\textwidth]{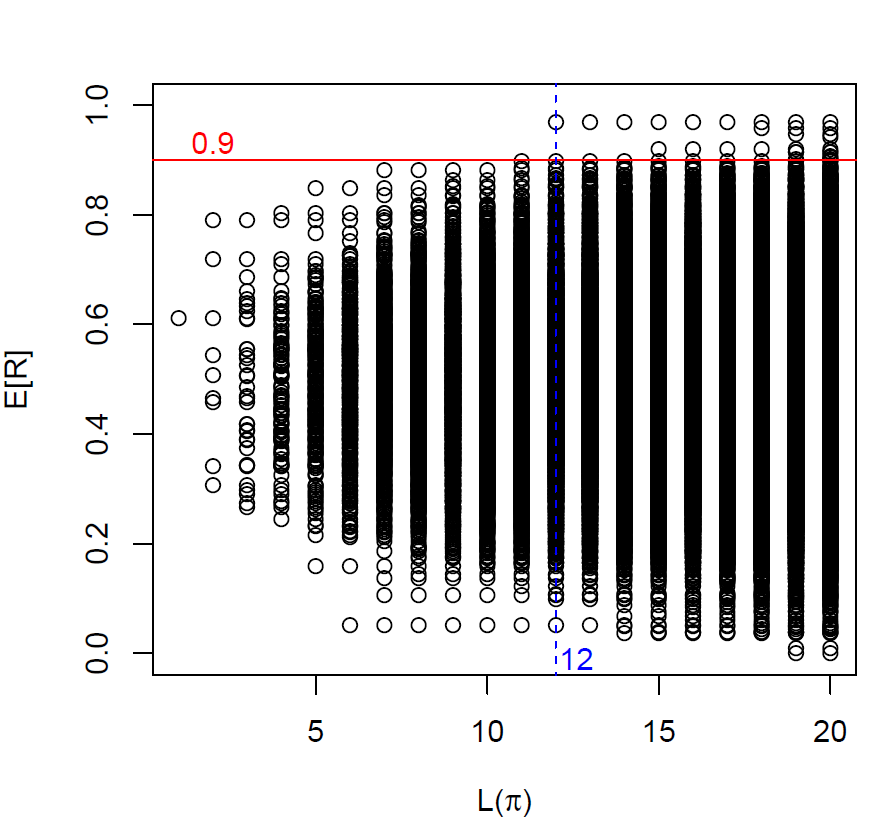}\vspace{0.5cm} \hfill 
\includegraphics[width=0.48\textwidth]{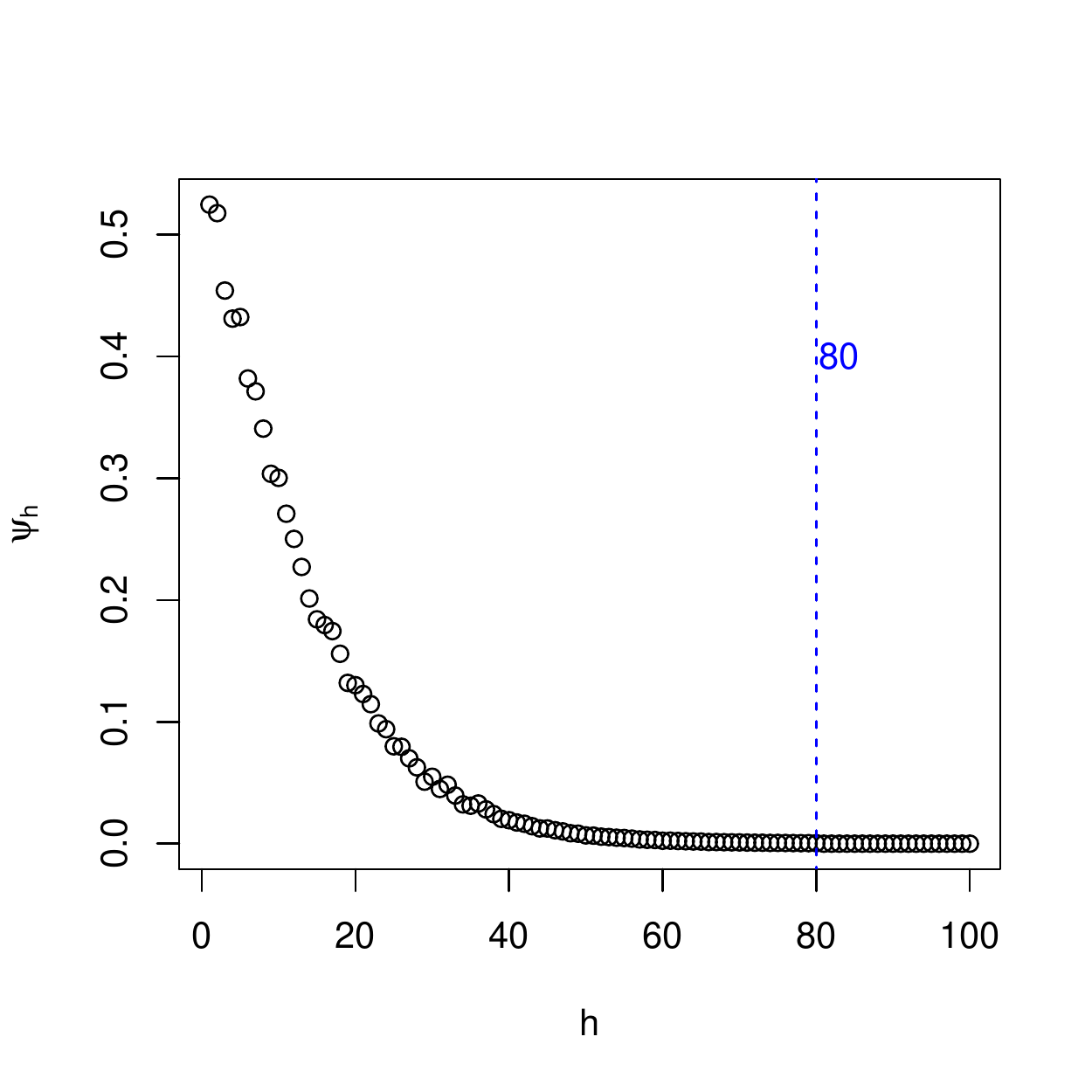} 
\vspace{-0.7cm}
\caption{Left: an illustrative distribution of responses of a population of agents for a single environment. 
If we set the threshold at $0.9 = 1 - \epsilon$, the simplest policy above this threshold is of `complexity' $h=12$. Right: an illustrative distribution of the result of $\Psydiff{h}$ considering a population of environments for a single agent, an {\em agent response curve}. There is no task $\pi$ of difficulty above 80 for which $\Exp{\Response(\pi, \mu \TAU)} \geq 1-\epsilon$, i.e., there is no task for which $\pi$ is $\epsilon$-acceptable, so $\Psydiff{h}$ is 0 from 80 on. If we were using the definition of $\hbar$ as for Eq.~\ref{eq:hbarK}, this 80 would be $\Length(\pi)$. Also note on the right plot that all `heaven' tasks (good results independently of what the agent does) are at $h=1$, while all `hell' tasks (bad response independently of what the agent does) are at $h=\infty$.}
\vspace{-0.3cm}
\label{fig:envdiff}
\end{figure}

With the difficulty function in eq.~\ref{eq:hbarK} we have:

\begin{proposition}\label{prop:bounded}
The difficulty function $\hbar^{[\epsilon]}$ in eq.~\ref{eq:hbarK} is strongly bounded.
\end{proposition}
\begin{proof}
For every policy $\pi$, if a task $\mu$ has a difficulty $\hbar^{[\epsilon]}(\mu) > \Length(\pi)$, it means that $\pi$ is not $\epsilon$-acceptable, because otherwise the difficulty would be $\Length(\pi)$ and not $h$. Consequently, $\Acc^{[\epsilon]}(\pi, \mu \TAU) = 0$ for all $\mu$ of difficulty $\hbar^{[\epsilon]}(\mu) > \Length(\pi)$. It is sufficient to take $h >\Length(\pi)$ for every $\pi$ to see that $\hbar$ is strongly bounded.
\end{proof}

This is what we see in Fig.~\ref{fig:envdiff} (right), where $\Length(\pi)=80$. With $\hbar^{[\epsilon]}$ in eq.~\ref{eq:hbarK}, we can ensure that the values are going to be 0 from $h=80$ on. 

This may not be the case for other difficulty functions. We can imagine a situation where the curve never converges to zero. 
For instance, if the difficulty function is decoupled from resources (length and/or steps) of the acceptable policies or we do not use the notion of $\epsilon$-acceptability then we cannot avoid that a very simple policy could eventually score well in a problem with very high difficulty. This would be counter-intuitive, as if there is a simple policy for a difficult problem, the latter should not be considered difficult any more.

\commentAGI{

Finally, it is an interesting question to determine what happens with this definition of difficulty when we consider not only the shortest policy but all the policies that are $\epsilon$-acceptable. 
Below, we just include one possibility, where we weight all the acceptable policies by some extreme distribution 
so that we can keep the original properties.

\begin{equation}\label{eq:hbarKmultiple}
\hbar^{[\epsilon]}(\mu) \triangleq -0.5\log \sum_{\pi} 2^{-2\Length(\pi)} \cdot \Acc^{[\epsilon]}(\pi, \mu \TAU)  
\end{equation}


\noindent The above expression is tweaked so that in the case there is only one acceptable policy, the function gives the same result as eq.~\ref{eq:hbarK} (if the sum only has one the case it reduces to $\Length(\pi)$). 
\comment{

# Some examples
-0.5*log2(2^(-2*10) )  # One of size 10
-0.5*log2(2^(-2*10)+ 2^(11) * 2^(-2*11) )  # One of size 10 and all of size 11
-0.5*log2(2^(-2*10)+ 2^(11) * 2^(-2*11) + 2^(12) * 2^(-2*12) )  # One of size 10 and all of size 11 and size 12

b <- 5
e <- 20
cum <- 0
for (i in b:e) {
  cum <- cum + (2^(i) * 2^(-2*i))  # Imagine all of length i are policies (and no prefix).
#  cum <- cum + (2^(-2*i))  # Imagine only on of length i is a policy
}
cum
2^(-(b-1))  # Note that this dominate cum
-0.5*log2(cum)/2

- log2(cum)  # >= b-1 = 4

}
In a way, this is a more robust version of difficulty than eq.~\ref{eq:hbarK}. Still, with the strong weighting ($2^{-2\Length(\pi)}$) we have used, we can see that strongly boundedness 
also holds for this version:

\begin{proposition}\label{prop:bounded2}
The difficulty function $\hbar^{[\epsilon]}$ in eq.~\ref{eq:hbarKmultiple} is strongly bounded.
\end{proposition}
\begin{proof}
Assume $k=\Length(\pi)$, the shortest acceptable policy for $\mu$. Let us consider one extreme (worst) case such that all agents of the same or more length are also acceptable. Since there are not more than $2^n$ programs of length $n$ (actually fewer if it is a prefix coding), we have $\hbar^{[\epsilon]} = -0.5 \log \sum_{n \geq k} 2^n2^{-2n} = -0.5 \log \sum_{n \geq k} 2^{-n} = -0.5 \log 2^{-k+1} = 0.5(k-1)$. 
For the other extreme (best) case, which is when we only have one acceptable policy for $\mu$, we have 
 $\hbar^{[\epsilon]} = -0.5 \log 2^{-2k} = k$. 
As this latter value is larger, we apply the same argument than for proposition \ref{prop:bounded} with $h > \Length(\pi)$.
\end{proof}

As the above difficulty function is more complex, 
 in what follows, we will work with the difficulty function $\hbar^{[\epsilon]}$ in eq.~\ref{eq:hbarK}.

}



\section{Difficulty-conditional task probabilities}\label{sec:probs}

In the previous sections we have focussed on $w(h)$ and whether it is necessary or not. We have seen difficulty functions where just aggregating $\Psydiff{h}$ without $w(h)$ (or $w(h)=1)$ leads to a $\Psy_{\oplus}(\pi,M,p_M \TAU)$ that is bounded\commentAGI{ (proposition \ref{prop:boundeddiff})}. The question now is how to choose the conditional probability $p_M(\mu|h)$. 
In the $C$-test, eq.~\ref{eq:ctest}, this was chosen as a uniform distribution. However, this is not possible in an interactive scenario if we consider all possible tasks, as the number of tasks for which there is an acceptable policy $\pi$ of $\Length(\pi)=n$ can be infinite. 
Even if we cannot set a uniform distribution, we want a choice of $p_M(\mu|h)$ that keeps the task diversity (unless there is any special bias to choose the tasks).


\subsection{Task probability depends on difficulty}\label{sec:probs1}

The first thing we can do is to assume $p(\mu|h)$ in eq. \ref{eq:UAcc} as 
$p(\mu|h) = \frac{2^{-K(\mu)}}{\nu(h)}$ if $\hbar^{[\epsilon]}(\mu) = h$ and 0 otherwise, where $\nu(h)$ is a normalisation term to make the mass of the distribution equal to 1, which can be formulated as $\nu(h) = \sum_{\mu : \hbar^{[\epsilon]}(\mu) = h} 2^{-K(\mu)}$.

And now we have:
\begin{eqnarray*}\label{eq:U2}
 \Psydiff{h}^{[\epsilon]}(\pi,M,p_M \TAU)  & =  & \sum_{\mu \in M, \hbar^{[\epsilon]}(\mu)= h} p_M(\mu|h) \cdot \Acc^{[\epsilon]}(\pi, \mu \TAU)  \commentAGI{\\
                   &} =  \commentAGI{&} \frac{1}{\nu(h)} \sum_{\mu \in M, \hbar^{[\epsilon]}(\mu)= h} 2^{-K(\mu)} \cdot \Acc^{[\epsilon]}(\pi, \mu \TAU)  
\end{eqnarray*}
\noindent From here, we can plug it into eq. \ref{eq:psi2-disc2} for the discrete case:
\begin{eqnarray}\label{eq:psi3-disc3}
\Psy_w^{[\epsilon]}(\pi,M,p_M \TAU) = \sum_{h=0}^{\infty} w(h) \frac{1}{\nu(h)} \sum_{\mu \in M, \hbar^{[\epsilon]}(\mu)= h} 2^{-K(\mu)} \cdot \Acc^{[\epsilon]}(\pi, \mu \TAU)  
\end{eqnarray}
\noindent Note that the above is going to be bounded independently of the difficulty function if $w$ is a probability distribution. 
Also notice that $\frac{1}{\nu(h)}$ is on the outer sum, and that $\nu(h)$ is lower than 1, so the normalisation term is actually greater than 1.

And if we use any of the difficulty functions in equations  \ref{eq:hbarK} \commentAGI{or \ref{eq:hbarKmultiple}} we can choose $w(h)=1$ and $\Psy_{\oplus}^{[\epsilon]}(\pi,M,p_M \TAU)$ is bounded.



\subsection{Task probability depends on the policy probability}\label{sec:probs2} 

\commentvisible{
Old material:
Sort by difficulty
All environments of difficulty h
\[ M_h(n, \theta) \triangleq \{ \mu : h(\mu, n, \theta)=h \} \]
This set is infinite.
Note that from our use of probabilistic machines, environnments for which random agents score better than the threshold 

have very low difficulty, which is natural.

** OPTION 1: WRONG
\[ \Upsilon_h(\pi, n, \theta) \triangleq \sum_{\mu in M_h(n, \theta)} p(\mu) R(\pi, \mu, n) \]
If $\mu$ is finite we could use this. But if $\mu$ is infinite we may do differently. Better below.

** OPTION 2: PROBLEMTIC...
For each difficulty consider all possible environments, but make it uniform for all possible "policies".

We define $\Pi_h(n)$ as the set of all agents $\pi$ with $C(\pi,\mu,n) = h$. WE WOULD NEED TO MAKE THE DEFINITION OF $C$ 

INDEPENDENT OF MU AND $\theta$!!! This set is finite.
Now we do a better option:
\[ \Upsilon_h(\pi, n, \theta) \triangleq \sum_{\mu in M_h(n, \theta, \pi' in Pi_h(n)} p(\mu|\pi') p(\pi') R(\pi, \mu, n) \]

$p(\pi) = 1 / |\Pi_h(n)|$, i.e., a uniform distribution. There can be some $\pi$ without any $\mu$!

And for each policy/agent, take a universal distribution (i.e., for all the environments that are solved by $\pi$ and 

not solved by (a significantly, i.e., unquestionability) shorter $\pi$), i.e.
\[p'(\mu|\pi) = 2^K(\mu) if \pi \in S*(\mu,n,\theta)\]
\[p'(\mu|\pi) = 0\] otherwise
and $p(\mu|\pi)$ is a normalisation of $p'$

** OPTION 3: RIGHT
We define 
\[Pair_h(n, \theta) \triangleq \{ <\mu, \pi> : \mu \in M_h and \pi \in S*(\mu, n, \theta) \}  \]
i.e., pairs or environments and policies (for that environment) such that $\mu$ has difficulty $h$ and $\pi$ gives that 

policy.
This set is infinite.
We can derive
\[\Pi_h(n, \theta) \triangleq \{ \pi : <\mu, \pi> \in Pair_h(n, \theta) \}\]
With this definition we now have a finite set of 

difficulty h.

\[\Upsilon_h(\pi, n, \theta) \triangleq \sum_{<\mu,\pi'> in Pair_h(n, \theta} p(\mu,\pi') R(\pi, \mu, n)\]

And now we work with $p(\mu,\pi')$ as follows:

\[p(\mu,\pi') = p(\mu|\pi') p(\pi')\]

where $p(\pi') =  1 / |\Pi_h(n)|$, i.e., a uniform distribution.
Now it is well-defined. And we go on as in option 2:

And for each policy/agent, take a universal distribution (i.e., for all the environments that are solved by $\pi$ and 

not solved by (a significantly, i.e., unquestionability) shorter $\pi$), i.e.
\[p'(\mu|\pi) = 2^K(\mu) if \pi \in S*(\mu,n,\theta)\]
\[p'(\mu|\pi) = 0\] otherwise
and $p(\mu|\pi)$ is a normalisation of $p'$.

In other words, we take all the simplest policies for at least one environment of difficulty $h$. For each of them we 

weight alls the environments for each of them using a universal distribution.

And now we can define:

\[\Upsilon(\pi, n, \theta) \triangleq \sum_{h in 0 to \infty} \Upsilon_h(\pi, n, \theta)\]

Prop. 1. Show that for every agent with finite C and for every $\theta > 0$ there is a difficulty h such that for every 

$h' > h$, we have that $\Upsilon_h(\pi, n, \theta) = 0$
Proof. It is sufficient to choose $h = C$. Consider any $h'>h$, if they could get some result above $\theta$, then the 

environment would not be in difficulty $h'$, but in $C$.

Corollary. For every agent $\pi$ with finite $C$, $\Upsilon$ is finite.

The measure is finite for any (resource-bounded) agent)
We sum for all thresholds. If rewards go between 0 and 1, we can use a uniform distribution.
}

One of things of the use of equation \ref{eq:hbarK} is that the number of acceptable policies per difficulty is finite. This is what happened in the $C$-test and that is the reason why a uniform distribution could be used for the inner sum. We could try to decompose the inner sum 
 by using the policy and get the probability of the task given the policy. 

\commentAGI{
We first need to define this set:
\begin{eqnarray*}\label{eq:pair}
Pairs(M,\Pi \TAU,\epsilon) & \triangleq & \{ \left\langle \mu, \pi \right\rangle : \mu \in M, \pi \in \Pi, \mbox{and} \: \pi \: \mbox{is an $\epsilon$-acceptable policy for} \: \mu \} \\
     & = & \{ \left\langle \mu, \pi \right\rangle : \mu \in M, \pi \in \Pi, \mbox{and} \: \Acc^{[\epsilon]}(\pi, \mu \TAU)  = 1 \}
\end{eqnarray*}
\noindent Note that one task can have many acceptable policies (and perhaps none) and one agent can be an acceptable policy of many tasks (and perhaps none).
Now we can have an alternative expression of difficulty equivalent to eq.~\ref{eq:hbarK} as follows:
\begin{eqnarray*}\label{eq:diff1}
\hbar^{[\epsilon]}(\mu) = \min_{\left\langle \mu, \pi \right\rangle \in Pairs(M,\Pi \TAU,\epsilon)} \Length(\pi) 
\end{eqnarray*}
\noindent We want to set $p_M(\mu|h)$ by decomposing it into $p_M(\mu|\pi)$ and $p_M(\pi|h)$.
\begin{eqnarray*}\label{eq:U3}
 \Psydiff{h}^{[\epsilon]}(\pi,M,p_M \TAU)  & =  & \sum_{\mu \in M, \hbar^{[\epsilon]}(\mu)= h} p_M(\mu|h) \cdot \Acc^{[\epsilon]}(\pi, \mu \TAU)  \\
                                     & =  & \sum_{\pi' : K(\pi) = h} p_M(\pi'|h) \sum_{\mu \in M : \left\langle \mu, \pi' \right\rangle \in Pairs(M,\Pi \TAU,\epsilon)} p_M(\mu|\pi') \cdot\Acc^{[\epsilon]}(\pi, \mu \TAU)
\end{eqnarray*}
\noindent If we use the difficulty function in equation \ref{eq:hbarK} we can assume that $p_M(\pi'|h)$ is uniform. If $N_\Pi(h)$ denotes the number of programs $\pi'$ for which $\Length(\pi')= h$, we can assume $p_M(\pi'|h) = \frac{1}{N(h)}$. So we can rewirte
\begin{eqnarray*}\label{eq:U3b}
 \Psydiff{h}^{[\epsilon]}(\pi,M,p_M \TAU)  & =  & \sum_{\pi' : K(\pi') = h}\frac{1}{N_\Pi(h)} \sum_{\mu \in M : \left\langle \mu, \pi' \right\rangle \in Pairs(M,\Pi \TAU,\epsilon)} p_M(\mu|\pi') \cdot \Acc^{[\epsilon]}(\pi, \mu \TAU)
\end{eqnarray*}
\noindent Now we may have the temptation to define $p_M(\mu|\pi')$ with a conditional universal distribution, i.e., $2^{-K(\mu|\pi')}$. There is a conundrum here, as we would have that those environments that can be easily described from the policy would be more likely. Even if $K$ is not symmetric there are strong relations in both directions such that the notion of difficulty would be completely mangled. Instead, it seems more reasonable to choose the distribution as $p_M(\mu|\pi') = \frac{2^{-K(\mu)}}{\nu(\pi')}$ if $\left\langle \mu, \pi' \right\rangle \in Pairs(M,\Pi \TAU)$ and 0 otherwise, where $\nu(\pi')$ is a normalisation term, which can be calculated as $\nu(\pi') = \sum_{\left\langle \mu, \pi' \right\rangle \in Pairs(M,\Pi \TAU)} 2^{-K(\mu)}$.

We can integrate the above into eq. \ref{eq:psi2-disc2} again. 
\begin{eqnarray}
\Psy_w^{[\epsilon]}(\pi,M,p_M \TAU) & = &\sum_{h=0}^{\infty} w(h) \sum_{\pi : K(\pi) = h}\frac{1}{N_\Pi(h)} \sum_{\mu \in M : \left\langle \mu, \pi' \right\rangle \in Pairs(M,\Pi \TAU,\epsilon)} \frac{2^{-K(\mu)}}{\nu(\pi')} \cdot \Acc^{[\epsilon]}(\pi, \mu \TAU)  \nonumber \\
 & = &\sum_{h=0}^{\infty}\frac{w(h)}{N_\Pi(h)}  \sum_{\pi : K(\pi) = h} \frac{1}{\nu(\pi')}\sum_{\mu \in M : \left\langle \mu, \pi' \right\rangle \in Pairs(M,\Pi \TAU,\epsilon)} 2^{-K(\mu)} \cdot \Acc^{[\epsilon]}(\pi, \mu \TAU) \label{eq:psi4}
\end{eqnarray}
\noindent} 
The interpretation would be as follows: for each difficulty value we aggregate all the acceptable policies with size equal to that difficulty uniformly and for each of these policies all the environments where each policy is acceptable with a universal distribution. 
This extra complication with respect to eq. \ref{eq:psi3-disc3} can only be justified if we generate environments and agents and we check them as we populate $Pairs$, as a way of constructing a test more easily. \commentAGI{Once a sample of $Pairs$ is generated we may want to see how to arrange everything and give appropriately weights to each case.

Apart from this, eq. \ref{eq:psi3-disc3} seems more natural and simpler than eq.~\ref{eq:psi4} as an extension of eq.~\ref{eq:ctest}.

}

\section{Using computational steps}\label{sec:steps}

As we mentioned in the introduction, the $C$-test \cite{HernandezOrallo-MinayaCollado1998,HernandezOrallo2000a} used Levin's $Kt$ instead of $K$. Apart from being computable, $Kt$ is related to Levin's search. This makes its choice much more appropriate for a measure of difficulty. However, when working with interactive tasks and with stochastic tasks and agents, we have that the number of computational steps must be calculated as $\Exp{\Steps(\pi,\mu \TAU)}$. 
For simplicity, up to this point, we have considered just $K$. Now we are briefly exploring the inclusion of the computational steps. 
\commentAGI{In section \ref{sec:notation}} \commentARXIV{In section \ref{sec:background}} we defined $\LS$. Now we define a version that accounts for the tolerance $\epsilon$ as follows:
\begin{eqnarray*}\label{eq:lsepsilon}
\LS^{[\epsilon]}(\pi,\mu \TAU) \triangleq \Exp{\LS(\pi,\mu \TAU)} \: \mbox{if} \: \Acc^{[\epsilon]}(\pi, \mu \TAU) = 1 \: \mbox{and}  \: \infty \: \mbox{otherwise} 
\end{eqnarray*}
\noindent and we define a new difficulty function that considers computational steps:
\begin{eqnarray*}\label{eq:hbarLS}
\hbar^{[\epsilon]}(\mu) \triangleq \min_\pi \LS^{[\epsilon]}(\pi,\mu \TAU)
\end{eqnarray*}
\noindent 
This difficulty function is not bounded, as $\LS$ depends on $\mu$, and we can always find a very short policy that 
takes an enormous amount of steps for a task with very high difficulty. This is an acceptable policy, but does not reduce the difficulty of the task, so it can always score non-zero beyond any limit. 

This means that for this difficulty function we would need to use equation eq.~\ref{eq:psi2-disc2} \commentAGI{or eq.~\ref{eq:psi2-cont2}} with an appropriate $w(h)$ (e.g., a small decay or a uniform interval of difficulties we are interested in).

 %
%
%
If the testing procedure established a limit on the number of steps (total or per transition) we would have this new difficulty function would be strongly bounded. 
Alternatively, we could reconsider the inclusion the computational steps in the notion of acceptability.  
In this case, the approach in section~\ref{sec:probs2} could not be used, as the probability of $\pi$ given $h$ would also depend on $\mu$. 
%
%
%


\commentvisible{

What about $\tau$?

The set of simplest policies giving 

\[ S^*(\mu \TAU, \epsilon) \triangleq \argmin_\pi \LS^{[\epsilon]}(\pi,\mu \TAU) \]

Note that $S^*$ can have equivalent formulations of a policy with the same length and execution time. Two alternative ways 
of doing the same thing with the same program length.
We define $s*(\mu,n, \theta)$ as the first lexicographically in $S*(\mu,n, \theta)$. Do we need it?

Stability: if the score is the average, look for a value of $t$ such that for every $\tau' > \tau$
\[ \argmin_{\tau} \LS^{[\epsilon]}(\pi,\mu \TAU) >= \theta] =  \argmin_{\tau'} \LS^{[\epsilon]}(\pi,\mu \TAU)' \]
i.e.,
$S*(\mu, t', \theta) = S*(\mu, t, \theta)$

i.e., the simplest policy above the threshold for $t$ is the simplest policy forever. This is known as the simplest policy for $\mu$.

Proposition: for every finite (and deterministic?) $\mu$ (considering tasks and agents is memory and space bounded) there is a $t$ from which it is stable.
Proof: $\mu$ and $\pi$ will have finite states and we can explore every $\pi$ until the state for $\mu$ and $\pi$ are both in the same state. Define a policy $\pi$ that executes $\mu$ with every possible $\pi$ and evaluate it... (when we find the first $\pi$ above the threshold we only need to check simpler programs). 
Stability is related to ergodicity. It is a better alternative but clearly not equivalent. Why? If the environment has a heaven (or an episodic goal) that can be achieved with $\pi$, then if it surpass the threshold with t then it will surpass the threshold with any t'>=t. If the environment has an unavoidable hell, then it is not stable. Of course if the heaven is unavoidable, $LST(\pi)$ is very low.
If we liked something more similar to ergodicity, the average of $R$ should be made with a sliding window and with a preceding random agent. 

Also unquestionability? Can be defined, but there is no need for that with environments... (interactive tasks)

}

\section{Discussion}\label{sec:discussion}

We have gone from eq.~\ref{eq:ctest} taken from $C$-test to eq.~\ref{eq:psi2-disc2} when using discrete difficulty functions \commentAGI{or to eq.~\ref{eq:psi2-cont2} when using continuous difficulty functions}. We have seen that difficulties allow for a more detailed analysis of what happens for a given agent, depending on whether it succeeds at easy or difficult tasks. For some difficulty functions\commentAGI{ (e.g., those for which proposition \ref{prop:boundeddiff} is true)}, we do not even need to determine the weight for each difficulty and just calculate the area, as an aggregated performance for all difficulties. Actually, given the characteristics of most of the difficulty functions we have seen, setting a maximum difficulty could be enough, for instance, 
$h=200$, more than convenient, as it is really challenging to look into a space of policies of 200 bits (without prior knowledge). So, a finite range of difficulties would be enough for the evaluation of feasible agents.

The important thing is that now we do not need to set an a priori distribution for all tasks $p(\mu)$, but just a conditional distribution $p(\mu|h)$. Note that if we set a high $h$ we have the freedom to find the simplest task that creates that difficulty. In fact, this is desirable, as when we try to create a difficult game with a small number of rules. 
Actually, the choice of $p(\mu|h)$ as a universal distribution still depends on the reference machine and can set most of the probability mass on smaller tasks, but as it is conditional on $h$, all trivial, dead or simply meaningless tasks have usually very extreme values of $h$ (very low or infinite). That means that there is a range of {\em intersting} difficulties, discarding very small values of $h$ and very large values of $h$. Figure~\ref{fig:Ctest2} is a nice example of this, where only difficulties between 1 and 8 were used, and we see also that $h=1$ and $h>7$ are not really very discriminating. The bulk of the testing effort must be performed in this range.

\commentAGI{
Now, let us spare a comment about the practical feasibility of generating a test that is conceived in terms of difficulty, which is determined from the length of the policy. It is pertinent to quote a piece from \cite{leggveness2013approximation}: ``Another important difference in our work is that we have directly sampled from program space. This is analogous to the conventional construction of the Solomonoff prior, which samples random bit sequences and treats them as programs. With this approach all programs that compute some environment count towards the environment's effective complexity, not just the shortest, though the
shortest clearly has the largest impact. This makes AIQ very efficient in practice since we can just run sampled programs directly, avoiding the need to have to compute complexity values through techniques such as brute force program search. For example, to compute the complexity of a 15 symbol program, the C-test required the execution of over 2 trillion programs. For longer programs,
such as many that we have used in our experiments, this would be completely intractable. One disadvantage of our approach, however, is that we never know the complexity of any given environment; instead we know just the length of one
particular program that computes it.''. While I agree that the calculation of difficulties may be hard, a couple of things must be said. First, yet again the previous quote talks about the {\em complexity of an environment}, which is based on its description (or on its Solomonoff prior), but not on its policies. Second, it is always preferrable to devote time to design good tests that lead to a few tasks that are discriminating and whose difficulty we are sure of than to have a test that can be obtained more efficiently but the agent will take thousands of useless tasks. In other words, we want tests that are practical when they are applied. We can devote much more resources (and time) for the offline generation of appropriate tasks and the calculation of their difficulty. This is actually what real measurement disciplines do, such as psychometrics, which devote most of the effort to get good tasks and do not hesitate to discard those tasks whose difficulty is dubious or are not discriminating. 
Effort in the assessment and classification of exercises pays off in more efficient evaluations.
}

\commentAGI{Finally, there is an important question about the choice of a meaningful difficulty function linked to the effort required to find an acceptable policy: what if difficulty depends on the verification of the policy? This has got attention for problems where the policy must be accompanied by a verification, proof or explanation \cite{HernandezOrallo00b,alpcan2014}. In other words, the question is whether we can extend Levin's search by considering the verification of stochastic problems and derive a difficulty function from it. This is being addressed in a separate note \cite{Upsychotasksarxiv}.}



\bibliography{biblio}

\end{document}